\documentclass{article}

\usepackage{microtype}
\usepackage{graphicx}
\usepackage{amssymb}
\usepackage{amsmath}
\usepackage{amsthm}
\usepackage{subfigure}
\usepackage[algo2e,ruled,vlined,linesnumbered]{algorithm2e}
\usepackage{booktabs} % for professional tables

\usepackage{hyperref}

\newtheorem{theorem}{Theorem}[section]

\usepackage[accepted]{icml2020}

\icmltitlerunning{Bidirectional Model-based Policy Optimization}

\begin{document}

\twocolumn[
\icmltitle{Bidirectional Model-based Policy Optimization}



\icmlsetsymbol{equal}{*}

\begin{icmlauthorlist}
\icmlauthor{Hang Lai}{sjtu}
\icmlauthor{Jian Shen}{sjtu}
\icmlauthor{Weinan Zhang}{sjtu}
\icmlauthor{Yong Yu}{sjtu}
\end{icmlauthorlist}

\icmlaffiliation{sjtu}{Shanghai Jiao Tong University, Shanghai, China}

% \icmlaffiliation{to}{Department of Computation, University of Torontoland, Torontoland, Canada}
% \icmlaffiliation{goo}{Googol ShallowMind, New London, Michigan, USA}
% \icmlaffiliation{ed}{School of Computation, University of Edenborrow, Edenborrow, United Kingdom}

\icmlcorrespondingauthor{Hang Lai}{laihang@apex.sjtu.edu.cn}
\icmlcorrespondingauthor{Weinan Zhang}{wnzhang@sjtu.edu.cn}

\icmlkeywords{Machine Learning, ICML}

\vskip 0.3in
]



\printAffiliationsAndNotice{}  % leave blank if no need to mention equal contribution

\begin{abstract}

Model-based reinforcement learning approaches leverage a forward dynamics model to support planning and decision making, which, however, may fail catastrophically if the model is inaccurate. Although there are several existing methods dedicated to combating the model error, the potential of the single forward model is still limited. In this paper, we propose to additionally construct a backward dynamics model to reduce the reliance on accuracy in forward model predictions. We develop a novel method, called Bidirectional Model-based Policy Optimization (BMPO) to utilize both the forward model and backward model to generate short branched rollouts for policy optimization. Furthermore, we theoretically derive a tighter bound of return discrepancy, which shows the superiority of BMPO against the one using merely the forward model. Extensive experiments demonstrate that BMPO outperforms state-of-the-art model-based methods in terms of sample efficiency and asymptotic performance.

\end{abstract}

\section{Introduction}
Reinforcement learning (RL) methods are commonly divided into two categories: (1) model-free RL (MFRL) which directly learns a policy or value function from observation data, and (2) model-based RL (MBRL) which builds a predictive model of the environment and generates samples from it to derive a policy or a controller. While MFRL algorithms have achieved remarkable success in different ranges of areas \cite{mnih2015human, lillicrap2015continuous, schulman2017proximal}, they need a large number of samples, which limit their applications mostly to simulators. Model-based methods, in contrast, have shown great potential in reducing the sample complexity \cite{deisenroth2013survey}.
However, the asymptotic performance of MBRL methods often lags behind their model-free counterparts due to model error, which is especially severe for multi-step rollout because of the compounding error \cite{asadi2018lipschitz}.

Several previous works have been proposed to alleviate compounding model error in different ways. For example,  \citet{whitney2018understanding} and \citet{asadi2019combating} introduced multi-step models which directly predict the consequence of executing a sequence of actions; \citet{mishra2017prediction} tried to divide trajectories into temporal segments and make predictions over segments instead of one timestep; \citet{talvitie2017self}, \citet{kaiser2019model} and \citet{Freeman2019LearningTP} proposed to train the dynamics model on its own outputs,
hoping the model is capable of giving accurate predictions in the region of its outputs. 
\citet{wu2019model} attributed long horizon rollouts mismatching to the one-step supervised learning and proposed to learn the model via model imitation. Besides the above innovations in model learning,
\citet{nguyenimproving} and \citet{ xiao2019learning} developed methods of using adaptive rollout horizon according to the estimated compounding error, and \citet{janner2019trust} proposed to use truncated short rollouts branched from real states. 
However, even with the compounding model error mitigated, it is almost impossible to obtain perfect multi-step predictions in practice, and thus the potential of the single forward model is still limited. 

In this paper, we investigate how to maximally take advantage of the underlying dynamics of the environment to reduce the reliance on the accuracy of the forward model. When making decisions, human beings can not only predict future consequences of current behaviors forward but also imagine possible traces leading to a target goal backward. Based on this consideration, inspired by \citet{edwards2018forward} and \citet{goyal2018recall}, we propose to learn a backward model in addition to the forward dynamics model to generate simulated rollouts backwards. To be specific, given a state and a previous action, the backward model predicts the preceding state, which means it can generate trajectories that terminate at one certain state.\

Now, intuitively, if we utilize both the forward model and the backward model to sample bidirectionally from some encountering states for policy optimization, the model error will compound separately in different directions. As Figure~\ref{fig:rollout comparison} shows, when we generate trajectories with the same length using a forward model and bidirectional models respectively, the compounding error of the latter will be less than that of the former. More theoretical analysis is 
given in Section \ref{sec:theory}.

\begin{figure}[tb]
    \centering
    \includegraphics[width=0.48\textwidth]{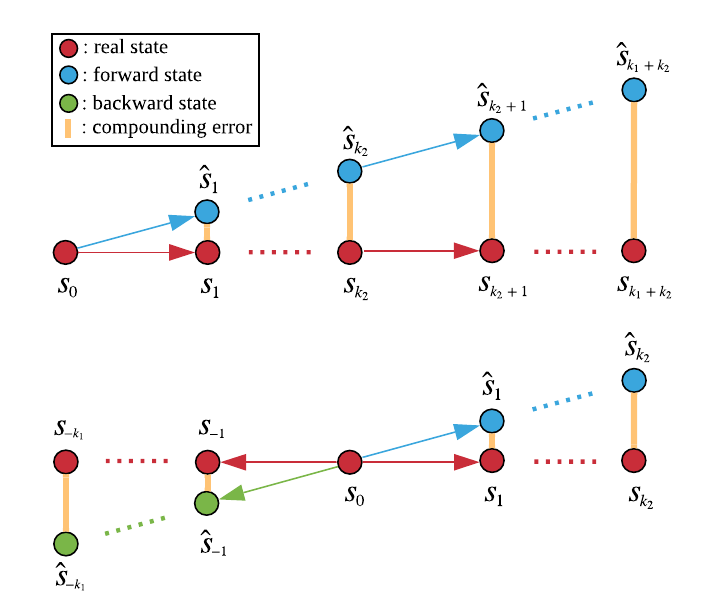}
    \vspace{-20pt}
    \caption{Comparison of compounding error of trajectories generated using only a forward model (top) and bidirectional models (bottom). States in real trajectories are shown in red, and states predicted by the forward (backward) model are shown in blue (green). The orange lines between real states and simulated states roughly represent the error, and when the rollout length increases, the error will compound. For bidirectional models, model error compounds $k_1$ steps backward and $k_2$ steps forward while for the forward model, model error compounds $k_1+k_2$ steps.}
    \label{fig:rollout comparison}
\end{figure}

With this insight, we combine bidirectional models with recent MBPO method \cite{janner2019trust} and propose a practical MBRL algorithm called Bidirectional Model-based Policy Optimization (BMPO).
Besides, we develop a novel state sampling strategy to sample states for model rollouts and incorporate model predictive control (MPC) \cite{camacho2013model} into our algorithm for further performance improvement. Furthermore, we theoretically prove that BMPO has a tighter bound of return discrepancy compared against MBPO, which verifies the advantage of bidirectional models in theory. We evaluate our BMPO and previous state-of-the-art algorithms \cite{haarnoja2018soft,janner2019trust} on a range of continuous control benchmark tasks. Experiments demonstrate that BMPO achieves higher sample efficiency and better asymptotic performance compared with prior model-based methods which only use forward models.

\section{Related Work}

Model-based reinforcement learning methods are expected to reduce sample complexity by learning a model as a simulator of the environment \cite{sutton2018reinforcement}. However, model error tends to cripple the performance of model-based approaches, which is also known as model-bias \cite{deisenroth2011pilco}. As is discovered in previous works \cite{venkatraman2015improving, talvitie2017self, asadi2018lipschitz}, even small model error can severely degrade multi-step rollouts since the model error will compound, and the predicted states will move out of the region where the model has high accuracy after a few steps \cite{asadi2018towards}.

To mitigate the compounding error problem, multi-step models \cite{whitney2018understanding,  asadi2019combating} were proposed to predict the outcome of executing a sequence of actions directly. Segments-based models have also been considered to make stable and accurate predictions over temporal segments \cite{mishra2017prediction}. Alternatively, a model may also be trained on its own outputs, hoping that the model can perform reasonable prediction in the region of its outputs \cite{talvitie2017self, kaiser2019model, Freeman2019LearningTP}. Besides, model imitation tried to learn the model by matching the multi-step rollouts distributions via WGAN \cite{wu2019model}. Furthermore, adaptive rollout horizon techniques were investigated to stop rolling out based on compound model error estimates \cite{nguyenimproving, xiao2019learning}. The MBPO algorithm \cite{janner2019trust} avoided the compounding error by generating short branched rollouts from real states.  
This paper is mainly based on MBPO framework backbone and allows for extended rollouts with less compounding error by generating rollouts bidirectionally.

There are many model architecture choices, such as linear models \cite{parr2008analysis, sutton2012dyna, levine2013guided, levine2014learning, kumar2016optimal}, nonparametric Gaussian processes \cite{kuss2004gaussian, ko2007gaussian, deisenroth2011pilco}, and neural networks \cite{draeger1995model, gal2016improving, nagabandi2018neural}. Model ensembles have shown to be effective in preventing a policy or a controller from exploiting the inaccuracies of any single model \cite{rajeswaran2016epopt,chua2018deep, kurutach2018model, janner2019trust}. In this paper, we adopt the ensemble of probabilistic networks for both forward and backward models.

Model predictive control (MPC) \cite{camacho2013model} is considered as an efficient and robust way in model-based planning, which utilizes the dynamic model to look forward several steps and optimizes the actions sequence over a finite horizon. \citet{nagabandi2018neural} combined deterministic neural networks model with a simple MPC of random shooting and acquired stable and plausible locomotion gaits in high-dimensional tasks. \citet{chua2018deep} improved this method by capturing two kinds of uncertainty in modeling and replacing random shooting with the cross-entropy method (CEM). \citet{wang2019exploring} further expanded MPC algorithms by applying policy networks to generate action sequence proposals. In this paper, we also incorporate MPC into our algorithm to refine the actions sampled from policy when interacting with the environment. 

Theoretical analysis for tabular- and linear-setting model-based methods has been conducted in prior works \cite{szita2010model,dean2017sample}. As for non-linear settings, \citet{sun2018dual} provided a convergence analysis for the model-based optimal control framework. \citet{luo2018algorithmic} built a lower bound of the expected reward while enforced a trust region on the policy. \citet{janner2019trust}, instead, derived a bound of discrepancy between returns in real environment and those under the branched rollout scheme of MBPO in terms of rollout length, model error, and policy shift divergence. In this paper, we also provide a theoretical analysis of leveraging bidirectional models in MBPO and obtain a tighter bound.

To facilitate research in MBRL, \citet{langlois2019benchmarking} gathered a wide collection of MBRL algorithms and benchmarked them on a series of environments specially designed for MBRL. Beyond comparing different methods, they also raised several crucial challenges for future MBRL research.

This work is closely related to \citet{goyal2018recall}, which built a backtracking model of the environment and used it to generate traces leading to high value states. These traces are then used for imitation learning to improve the model-free learner. Our approach differs from theirs in the following two aspects: (1) their method is more like model-free RL as they mainly train the policy on real observation data and use the model generated traces to fine-tuning the policy, while our method is purely model-based RL and optimize the policy with model-generated data. (2) they only build a backward model of the environment, whereas we build a forward model and a backward model at the same time.

\section{Preliminaries}
\subsection{Reinforcement Learning}
A discrete-time Markov decision process (MDP) with infinite horizon is defined by the tuple ($\mathcal{S}, \mathcal{A},M^*, r, \gamma, \rho_{0}$). Here, $\mathcal{S} \in \mathbb{R}^{d_{s}}$ and $\mathcal{A} \in \mathbb{R}^{d_{a}}$ denote the state and action spaces, respectively.  $\gamma \in [0,1)$ is the discount factor. $M^*(s'|s,a)$ denotes the transition distribution function, and $r(s,a)$ denotes the reward function. $\rho_{0}$ represents the initial state distribution.
 Let $\eta$ denote the expected return, \textit{i.e.}, the sum of the discounted rewards along a trajectory $\tau:=\left(s_{0}, a_{0}, \dots, s_{t}, a_{t} \dots \right)$. The goal of reinforcement learning is to find the optimal policy $\pi^*$ that maximizes the expected return:
\begin{equation}
\label{eq: rl-obj}
    \pi^* = \mathop{\arg \max}_\pi \eta[\pi]=\mathop{\arg \max}_\pi \mathbb{E}_\pi \Big[\sum_{t=0}^\infty \gamma^t r(s_t,a_t) \Big].
\end{equation}
where $s_{0} \sim \rho_{0}$, \ $a_{t} \sim \pi(a_{t}|s_t)$, \ $s_{t+1} \sim M^*(s_{t+1}|s_t,a_t)$.

\subsection{Dyna-style Algorithm}

Dyna-style algorithm \cite{sutton1991dyna, langlois2019benchmarking} is one kind of MBRL method that uses a model for policy optimization. To be specific, in Dyna, a forward dynamic model $p_{\theta}\left(s^{\prime} | s, a\right)$ is learned from agent interactions with the environment by executing the policy $\pi_{\phi}\left(a | s\right)$. The policy $\pi_{\phi}$ is then optimized using model-free algorithms on real data and model-generated data. To perform the Dyna-style algorithm using bidirectional models, we further learn a backward model $q_{\theta^{\prime}}\left(s | s^{\prime}, a\right)$ and a backward policy $\tilde{\pi}_{\phi^{\prime}}\left(a | s^{\prime}\right)$. As such, forward rollouts and backward rollouts can be generated simultaneously for policy optimization.  

\begin{figure*}[htb]
    \centering
    \vspace{-10pt}
    \includegraphics[width=1.0\textwidth]{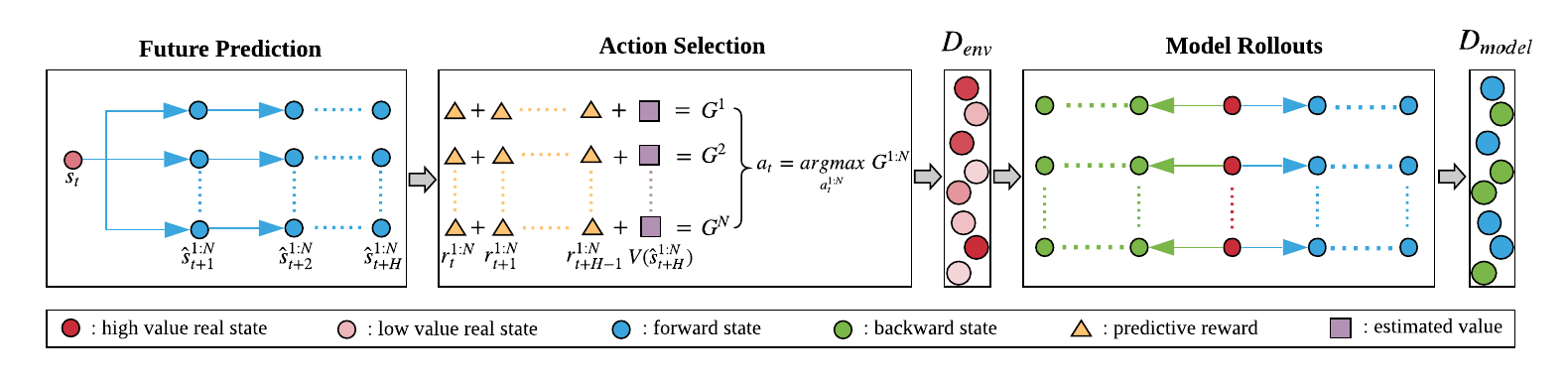}
    \vspace{-25pt}
    \caption{Overview of the BMPO algorithm. When interacting with the environment, the agent uses the model to predict the future states over a finite horizon $H$ and then take the first action of the sequence with the highest simulated estimated return. The transitions are stored in $\mathcal{D}_{\mathrm{env}}$, where the state value is represented by the shade of color (value increases from light red to dark red). High value states are then sampled from $\mathcal{D}_{\mathrm{env}}$ to perform bidirectional model rollouts, which are stored in $\mathcal{D}_{\mathrm{model}}$ for policy optimization afterward.}
    \vspace{-7pt}
    \label{fig:alg_framework}
\end{figure*}

\section{BMPO Framework}

\begin{algorithm2e}[t]
\caption{Bidirectional Model-based Policy Optimization (BMPO)}
\LinesNumbered
\DontPrintSemicolon
\label{alg:BMPO}
Initialize policy $\pi_{\phi}$, backward policy $\tilde{\pi}_{\phi^{\prime}}$, forward model $p_{\theta}$, backward model $q_{\theta^{\prime}}$, environment replay buffer $\mathcal{D}_{\mathrm{env}}$, model replay buffer $\mathcal{D}_{\mathrm{model}}$\;
\For{$N$ epochs}{
Train model $p_{\theta}$ and $q_{\theta^{\prime}}$ using $\mathcal{D}_{\mathrm{env}}$ via applying gradient descent on Equation \ref{equation:forward model loss} and \ref{equation:backward model loss}\;
\For{$E$ steps}{
Take an action in environment according to MPC; add to $\mathcal{D}_{\mathrm{env}}$\;
\For{M model rollouts}{
Sample state $s_{t}$ from $\mathcal{D}_{\mathrm{env}}$ with probability computed via Equation \ref{eq: state choose possibility}\;
Perform $k_1$ steps backward rollouts and $k_2$ steps forward rollouts starting from $s_t$; add to $\mathcal{D}_{\mathrm{model}}$\;
}
\For{$G$ gradient updates}{
Update policy parameters on model data: $\phi \leftarrow \phi-\lambda_{\pi} \hat{\nabla}_{\phi} J_{\pi}\left(\phi, \mathcal{D}_{\text {model }}\right)$\;
}
\For{$G^{\prime}$ gradient updates}{
Update backward policy on $\mathcal{D}_{\mathrm{env}}$ by optimizing Equation \ref{equation:backward policy mle loss} or \ref{equation:backward policy gan loss}
}
}
}

%\vspace{-2pt}

\end{algorithm2e}

In this section, we introduce how BMPO leverages bidirectional models to generate more plentiful simulated data for policy optimization in detail. Although bidirectional models can be incorporated into almost any Dyna-style model-based algorithms \cite{sutton1991dyna}, we choose the Model-based Policy Optimization (MBPO) \cite{janner2019trust} algorithm as the framework backbone since it is the state-of-the-art MBRL method and is sufficiently general. The overall algorithm is demonstrated in Algorithm \ref{alg:BMPO}
%\footnote{\lh{In line 10, SAC will first update the Q-function and then optimize policy accordingly. We omit the former here for clarity.}}
, and an overview of the algorithm architecture is shown in Figure \ref{fig:alg_framework}.

\subsection{Bidirectional Models Learning in BMPO}

\subsubsection{Forward Model}
For the forward model, we use an ensemble of bootstrapped probabilistic dynamics models, which were first introduced in PETS \cite{chua2018deep} to capture two kinds of uncertainty. To be specific, individual probabilistic models can capture the aleatoric uncertainty aroused from the inherent stochasticity of a system, and the bootstrapped ensemble aims to capture the epistemic uncertainty due to the lack of sufficient training data. Prior works \cite{chua2018deep, janner2019trust} have demonstrated that the ensemble of probabilistic models is quite effective in MBRL, even when the ground truth dynamics are deterministic.

In detail, each in the forward model ensembles $\left\{p_{\theta}^{i}\right\}_{i=1}^B$ is parameterized by a multi-layer neural network, and we denote the parameters in the forward model ensembles as $\theta$. Given a state $s$ and an action $a$, each probabilistic neural network outputs a Gaussian distribution with diagonal covariances of the next state: $p_{\theta}^{i}\left(s^{\prime} | s, a\right)=\mathcal{N}\left({\mu}_{\theta}^{i}\left({s}, {a}\right), {\Sigma}_{\theta}^{i}\left({s}, {a}\right)\right)$.\footnote{The model network will output a similar Gaussian distribution of the reward as well, which is omitted here for simplicity.}  We train the model ensembles with different initializations and bootstrapped samples of the real environment data via maximum likelihood, and the corresponding loss of the forward model is
\begin{equation}
\begin{aligned} 
\label{equation:forward model loss}
\mathcal{L}_f({\theta})=&\sum_{t=1}^{N}\left[\mu_{\theta}\left({s}_{t}, {a}_{t}\right)-{s}_{t+1}\right]^{\top} {\Sigma}_{\theta}^{-1}\left({s}_{t}, {a}_{t}\right)
\\& \left[\mu_{\theta}\left({s}_{t}, {a}_{t}\right)-{s}_{t+1}\right]+\log \operatorname{det} {\Sigma}_{\theta}\left({s}_{t}, {a}_{t}\right),
\end{aligned}
\end{equation}
where $\mu$ and $\Sigma$ are the mean and covariance respectively, and $N$ denotes the total number of transition data.

\subsubsection{Backward Model}

Besides the traditional forward model learning, we additionally learn a backward model for extended simulated data generation. In this way, we can mitigate the error caused by excessively exploiting the forward model. In other words, the backward model is used to reduce the burden of the forward model on generating data.

Due to the powerful capabilities of the probabilistic network ensemble, similarly to the forward model, we adopt the same parameterization for the backward model. In detail, another multi-layer neural network with parameters $\theta^{\prime}$ is used to output a Gaussian prediction, \textit{i.e.}, $q_{\theta^{\prime}}^{i}\left(s | s^{\prime}, a\right)=\mathcal{N}\left({\mu}_{\theta^{\prime}}^{i}\left({s^{\prime}}, {a}\right), {\Sigma}_{\theta^{\prime}}^{i}\left({s^{\prime}}, {a}\right)\right)$, and the loss function of the backward model is
\begin{align}
\mathcal{L}_b({\theta^{\prime}})=&\sum_{t=1}^{N}\left[\mu_{\theta^{\prime}}\left({s}_{t+1}, {a}_{t}\right)-{s}_{t}\right]^{\top} {\Sigma}_{\theta^{\prime}}^{-1}\left({s}_{t+1}, {a}_{t}\right) \label{equation:backward model loss}\\
& \left[\mu_{\theta^{\prime}}\left({s}_{t+1}, {a}_{t}\right)-{s}_{t}\right]+\log \operatorname{det} {\Sigma}_{\theta^{\prime}}\left({s}_{t+1}, {a}_{t}\right). \nonumber
\end{align}

We note that although we can directly train the backward model with data sampled from the environment, it remains a problem how to choose the actions when we use the backward model to generate trajectories backwards. Recall that in the forward model situation, the actions are usually taken by the current policy $\pi_{\phi}\left(a | s\right)$ or an exploration policy. Thus we need an additional backward policy $\tilde{\pi}_{\phi^{\prime}}\left(a | s^{\prime}\right)$ to generate actions in the backward rollouts.

\subsection{Policy Optimization in BMPO}

\subsubsection{Backward policy}
\label{section:backward policy}
To sample trajectories backwards, we need to design a backward policy $\tilde{\pi}_{\phi^{\prime}}\left(a | s^{\prime}\right)$ to take action given the next state. There are many alternatives for backward policy design. In our experiments, we use two heuristic methods for comparison.

The first way is to train the backward policy by maximum likelihood estimation according to the $(a, s')$ data sampled in the environment. The corresponding loss function is 
\begin{equation}
\mathcal{L}_{\text{MLE}}({\phi^{\prime}})=-\sum_{t=0}^{N} \log \tilde{\pi}_{\phi^{\prime}}\left(a_{t} | s_{t+1}\right).
\label{equation:backward policy mle loss}
\end{equation}

The second way is to use a conditional generative adversarial network (CGAN) \cite{mirza2014conditional} to generate an action $a$ conditioned on the next state $s^{\prime}$, and the adversarial loss can be written as
\begin{align}
\min _{\tilde{\pi}} \max _{D} V(D, \tilde{\pi})=&\mathbb{E}_{(a,s^{\prime}) \sim \pi}\left[ \log D(a,s^{\prime})\right]+ \label{equation:backward policy gan loss} \\
&\mathbb{E}_{s^{\prime} \sim \pi}\left[\log \left(1-D(\tilde{\pi}(\cdot | s^{\prime}),s^{\prime})\right)\right], \nonumber
\end{align}
where $D$ is an additional discriminator and the generator here is the backward policy $\tilde{\pi}$.

%A good backward policy should be like an inverse of the current policy $\pi_\phi$ so that the inverse of the simulated backward trajectories will be like the forward trajectories.
Our goal is to make the backward rollouts resemble the real trajectory sampled by the current forward policy. Thus when training the backward policy, we only use the recent trajectories sampled by the agent in the real environment.

\subsubsection{MBPO with Bidirectional Models}
The original MBPO \cite{janner2019trust} algorithm iterates between three stages: data collection, model learning, and policy optimization. In the data collection stage, data is collected by executing the latest policy in real environment and added to replay buffer $\mathcal{D}_{\mathrm{env}}$. In the model learning stage, an ensemble of bootstrapped forward models are trained using all the data in $\mathcal{D}_{\mathrm{env}}$ through maximum likelihood. Then, in the policy optimization stage, short-length rollouts starting from states randomly chosen from $\mathcal{D}_{\mathrm{env}}$ are generated using the forward model. 
Model-generated data is then used to train a policy $\pi$ (see line 10 in Algorithm \ref{alg:BMPO}) through Soft Actor-Critic (SAC) \cite{haarnoja2018soft} by minimizing the expected KL-divergence $J_{\pi}(\phi, \mathcal{D})=\mathbb{E}_{s_{t} \sim \mathcal{D}}\left[D_{K L}\left(\pi \| \exp \left\{Q^{\pi}-V^{\pi}\right\}\right)\right]$.

Bidirectional models can be incorporated into MBPO naturally. Specifically, we iteratively collect data, train bidirectional models and backward policy, and then use them to generate short-length rollouts bidirectionally starting from some real states to optimize our policy. Besides, in Section \ref{sec:decisions} we design other components specifically for  BMPO, \emph{i.e.}, state sampling and MPC, which are crucial for the improvement of performance.

\subsection{Design Decisions}\label{sec:decisions}
\subsubsection{State Sampling Strategy}
\label{section:sampling strategy}
To better exploit the bidirectional models, inspired by \citet{goyal2018recall}, we begin simulated branched rollouts bidirectionally from high value states instead of randomly chosen states in the environment replay buffer. 
In such a way, the agent could learn to reach high value states through backward rollouts, and also learn to act better after these states through forward rollouts.

An intuitive idea is simply choosing the first $K$ highest-value states to begin rollouts, which, however, will cause the agent not knowing how to behave in low value states due to the absence of low value states data. For better generalization and stability, we choose starting states from the environment replay buffer $\mathcal{D}_{\mathrm{env}}$ according to Boltzmann distribution based on the value $V(s)$ estimated by SAC. Let $p(s)$ be the probability of a state $s$ being chosen as a starting state, then we have
\begin{equation}
\label{eq: state choose possibility}
p(s) \propto e^{\beta V(s)},
\end{equation}
where $\beta$ is the hyperparameter to control the ratio of high value states.

\subsubsection{Incorporating MPC}
In BMPO, when taking actions in the real environment, we adopt model predictive control (MPC) to exploit the forward model further. MPC is a common model-based planning method using the learned model to look forward and optimize actions sequence. In detail, at each timestep, $N$ candidate action sequences  $\mathbf{A}_{t}^{1:N}=\left(a_{t}^{1:N}, \ldots, a_{t+H-1}^{1:N}\right)$ are generated, where $H$ is the planning horizon, and the corresponding trajectories are simulated in the learned model. Then the first action of the sequence that yields the highest accumulated rewards is selected. Here, we use a variant of traditional MPC: (1) we generate action sequences from the current policy like \citet{wang2019exploring} instead of uniform distribution; (2) we add the value estimate of the last state in the simulated trajectory to the planning objective, \textit{i.e.}, 
\begin{equation}
\label{eq: mpc}
a_{t} = \underset{a_{t}^{1:N} \sim \pi} {\text{argmax}} \ \ \sum_{t^{\prime}=t}^{t+H-1} \gamma^{t^{\prime}-t}r\left(s_{t^{\prime}}, a_{t^{\prime}}\right) + \gamma^{H}V(s_{t+H}).
\end{equation}

It is worth noting that we only use MPC in the training phase but not in the evaluation phase. In the training phase, MPC helps select actions to take in real environments, and the obtained high value states will be used to generate simulated rollouts for policy optimization. In the evaluation phase, the trained policy is evaluated directly.

\section{Theoretical Analysis}
\label{sec:theory}
In this section, we theoretically analyze the discrepancy between the expected returns in the real environment and those under the bidirectional branched rollout scheme of BMPO. The proofs can be found in the appendix as provided in supplementary materials.
When we use the bidirectional models to generate simulated rollouts from some encountered state, we assume that the length of backward rollouts is $k_1$, and the length of forward rollouts is $k_2$.
Under such a scheme, we first consider a more general return discrepancy of two arbitrary bidirectional branched rollouts.
\newtheorem{lemma}{Lemma}[section]
\begin{lemma}
(Bidirectional Branched Rollout Returns Bound). Let $\eta_1$, $\eta_2$ be the expected returns of two bidirectional branched rollouts. Out of the branch, we assume that the expected total variation distance between these two dynamics at each timestep $t$ is bounded as $\max _{t} E_{(s,a) \sim p_{1}^{t}(s,a)} D_{TV}\left(p_{1}^{\mathrm {pre }}\left(s^{\prime} | s, a\right) \| p_{2}^{\mathrm {pre }}\left(s^{\prime} | s, a\right)\right) \leq \epsilon_{m}^{\mathrm{pre }}$, similarly, the forward branch dynamic bounded as $\max _{t} E_{(s,a) \sim p_{1}^{t}(s,a)} D_{TV}(p_{1}^{\mathrm {for }}\left(s^{\prime} | s, a\right) \| p_{2}^{\mathrm {for }}\left(s^{\prime} | s, a\right)) \leq \epsilon_{m}^{\mathrm {for }}$, and the backward branch dynamic bounded as $\max _{t} E_{(s^{\prime},a) \sim p_{1}^{t}(s^{\prime},a)} D_{TV}(p_{1}^{\mathrm {back }}\left(s | s^{\prime}, a\right) \| p_{2}^{\mathrm {back }}\left(s | s^{\prime}, a\right))$ $\leq \epsilon_{m}^{\mathrm {back }}$. Likewise, the total variation distance of policy is bounded by $\epsilon_{\pi}^{\mathrm{pre}}$, $\epsilon_{\pi}^{\mathrm{for}}$ and $\epsilon_{\pi}^{\mathrm{back}}$, respectively. Then the returns are bounded as:
\begin{align}
\left|\eta_{1}-\eta_{2}\right| \leq 2 r_{\max} &\Big[\frac{\gamma^{k_1+k_2+1}}{(1-\gamma)^{2}}\left(\epsilon_{m}^{\mathrm{pre}}+\epsilon_{\pi}^{\mathrm{pre}}\right)+\frac{\gamma^{k_1+k_2}}{1-\gamma}\epsilon_{\pi}^{\mathrm{pre}}
\nonumber \\
+&\frac{1-\gamma^{k_1}}{1-\gamma}\left(k_1\left(\epsilon_{m}^{\mathrm{back}}+\epsilon_{\pi}^{\mathrm{back}}\right)+\epsilon_{\pi}^{\mathrm{back}}\right) \nonumber \\ +&\frac{\gamma^{k_1}}{1-\gamma}\left(k_2\left(\epsilon_{m}^{\mathrm{for}}+\epsilon_{\pi}^{\mathrm{for}}\right)+\epsilon_{\pi}^{\mathrm{for}}\right)\Big].\label{eq:Bidiractional_branch_return_Bound}
\end{align}
\end{lemma}
\begin{proof}
\vspace{-4pt}
% See \emph{Appendix A}, \emph{Lemma A.1}. 
See Appendix \ref{appendix:Theorem}, Lemma \ref{lemma:Bidirectional Branched Rollout Returns Bound}. 
\vspace{-4pt}
\end{proof}

Now, we can bound the discrepancy between the returns in the environment and in the branched rollouts of BMPO. Let $\eta[\pi]$ be the expected return of executing current policy in the true dynamics, and $\eta^{\mathrm {branch }}[\pi]$ be the expected return of executing current policy in the model generated branch and executing old policy $\pi_{D}$ out of the branch. Then we can derive the discrepancy between them as follows.

\begin{theorem}
(BMPO Return Discrepancy Upper Bound) Assume that the expected total variation distance between the learned forward model $\hat{p}$ and the true dynamics $p$ at each timestep $t$ is bounded as 
$\max _{t} E_{(s,a) \sim \pi_{t}}\left[D_{T V}\left(p\left(s^{\prime} | s, a\right) \| \hat{p}\left(s^{\prime} | s, a\right)\right)\right] \leq \epsilon_{m}^{\mathrm{for}}$.
Similarly, the error of backward model $\hat{q}$ is bounded as 
$\max _{t} E_{(s^{\prime},a) \sim \pi_{t}}\left[D_{T V}\left(q\left(s | s^{\prime}, a\right) \| \hat{q}\left(s | s^{\prime}, a\right)\right)\right] \leq \epsilon_{m}^{\mathrm{back}}$ 
and the variation between current policy and the behavioral policy is bounded as 
$\max _{s} D_{T V}\left(\pi_{D}(a | s) \| \pi(a | s)\right) \leq \epsilon_{\pi}$.
Assume $\epsilon_{m}^{\mathrm{for}} \approx \epsilon_{m}^{\mathrm{back}} = \epsilon_{m}$, then under a branched rollouts scheme with a backward branch length of $k_1$ and a forward branch length of $k_2$, we have
\begin{align}
    \left|\eta[\pi]- \eta^{\mathrm {branch }}[\pi] \right| \leq & 2 r_{\max }\Big[\frac{\gamma^{k_1+k_2+1} \epsilon_{\pi}}{(1-\gamma)^{2}}
     \label{eq: BM_lower_Bound}
\\ & +\frac{\gamma^{k_1+k_2} \epsilon_{\pi}}{(1-\gamma)}+\frac{\max(k_1,k_2)\epsilon_{m}}{1-\gamma}\Big]. \nonumber
\end{align}
\end{theorem}
\begin{proof}
\vspace{-4pt}
% See \emph{Appendix A}, \emph{Theorem A.1}. 
See Appendix \ref{appendix:Theorem}, Theorem \ref{the:BMPO Return Discrepancy Upper Bound}.
\vspace{-4pt}
\end{proof}

We notice that in MBPO \cite{janner2019trust}, the authors derived a similar return discrepancy bound (refer to Theorem 4.3 therein) with only one forward dynamics model. Setting the forward rollout length as $k_1 + k_2$, the bound is
\begin{align}
   \left| \eta[\pi] - \eta^{\mathrm {branch }}[\pi] \right| \leq & 2 r_{\max }\Big[\frac{\gamma^{k_1+k_2+1} \epsilon_{\pi}}{(1-\gamma)^{2}} \label{eq: FM_low_Bound_k1_k2} \\ & +\frac{\gamma^{k_1+k_2} \epsilon_{\pi}}{(1-\gamma)}+\frac{(k_1+k_2)\epsilon_{m}}{1-\gamma}\Big]. \nonumber
\end{align}

By comparing the two bounds, it is evident that BMPO obtains a tighter upper bound of the return discrepancy by employing bidirectional models. The main difference is that in BMPO the coefficient of the model error is $\max(k_1,k_2)$, while in MBPO the coefficient is $k_1+k_2$. 
It is easy to understand: as we generate the branched rollouts bidirectionally, the model error will compound only $k_1$ steps in the backward model or $k_2$ steps in the forward model and thus at most $\max(k_1,k_2)$ steps,  
instead of $k_1+k_2$ steps using forward model only.

\section{Experiments}
Our experiments aim to answer the following three questions: 1) How does BMPO perform compared with model-free RL methods and previous state-of-the-art model-based RL methods using only one forward model? 2) Does using the bidirectional models reduce compounding error compared with using one forward model? 3) What are the critical components of our overall algorithm?

\begin{figure*}[htb]
	\centering
	\vspace{-4pt}
	\includegraphics[width=1\textwidth]{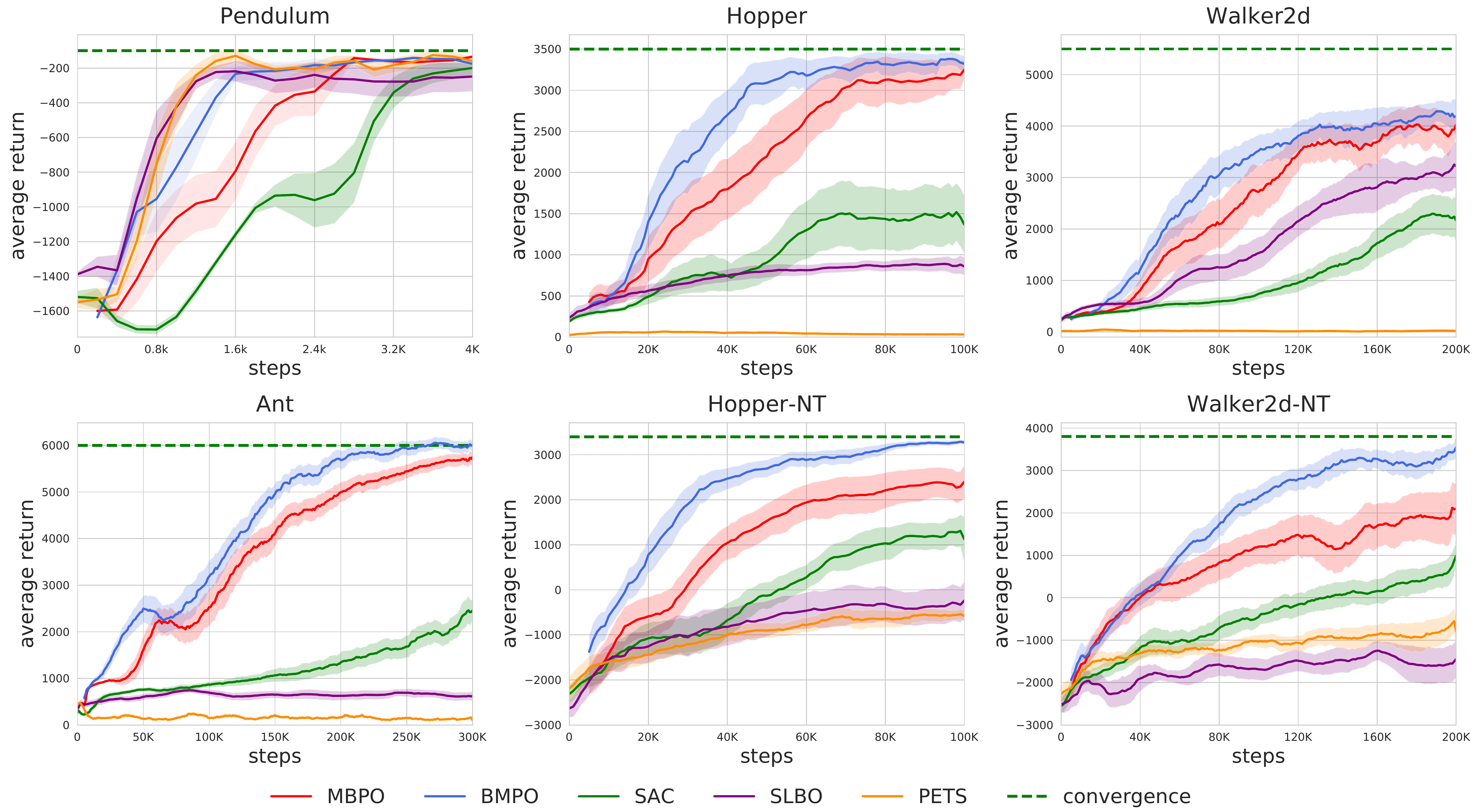}
	\vspace{-15pt}
	\caption{Learning curves of BMPO (ours) and four baselines on different continuous control environments. The solid lines indicate the mean and shaded areas indicate the standard error of six trails over different random seeds. Each trial is evaluated every 1000 environment steps (200 steps for Pendulum), where each evaluation reports the average return over ten episodes. The dashed reference lines are the asymptotic performance of SAC.}
	
	\vspace{-10pt}
	\label{fig:comparison result}
\end{figure*}

\subsection{Comparison with State-of-the-Arts}

In this section, we compare our method with previous state-of-the-art baselines. Specifically, for model-based methods, we compare against MBPO \cite{janner2019trust}, as our method builds on top of it; SLBO \cite{luo2018algorithmic} and PETS \cite{chua2018deep}, both performing well in the model-based benchmarking test \cite{langlois2019benchmarking}. For model-free methods, we compare to Soft Actor-Critic (SAC) \cite{haarnoja2018soft}, which is proved to be effective in continuous control tasks. For the MBPO baseline, we only generate $k_2$-steps forward rollout, where $k_2$ is set as the default value used in the original MBPO paper. We do not report the result of MBPO using rollouts of $k_1+k_2$ steps since using only $k_2$ steps achieves better performance. Notice that we do not include the backtracking model method \cite{goyal2018recall} for comparison since it is not a purely model-based method, and its performance improvement is limited compared with SAC.

We evaluate all the algorithms in six environments in total using OpenAI Gym \cite{brockman2016openai}\footnote{Code is available at: \href{https://github.com/hanglai/bmpo}{https://github.com/hanglai/bmpo}}. Among them, Pendulum is one traditional control task, and Hopper, Walker2D, Ant are three complex MuJoCo tasks \cite{todorov2012mujoco}. We additionally add two variants of MuJoCo tasks without early termination states, denoted as Hopper-NT and Walker2d-NT, which have been released as benchmarking environments for MBRL \cite{langlois2019benchmarking}. More details about the environments can be found in Appendix \ref{appendix:Env}.

The comparison results are shown in Figure \ref{fig:comparison result}. In different locomotion tasks,  our method BMPO learns faster and has better asymptotic performance than previous model-based algorithms using only the forward model, which empirically demonstrates the advantage of bidirectional models. This gap is even more significant in benchmarking environments Hopper-NT and Walker2d-NT. One possible reason is that in the environments without early termination, the space of encountered states is larger, and thus the introduced state sampling strategy is more effective.

\subsection{Model Error}

\begin{figure}[tb] 
    \centering
    \includegraphics[width=0.48\textwidth]{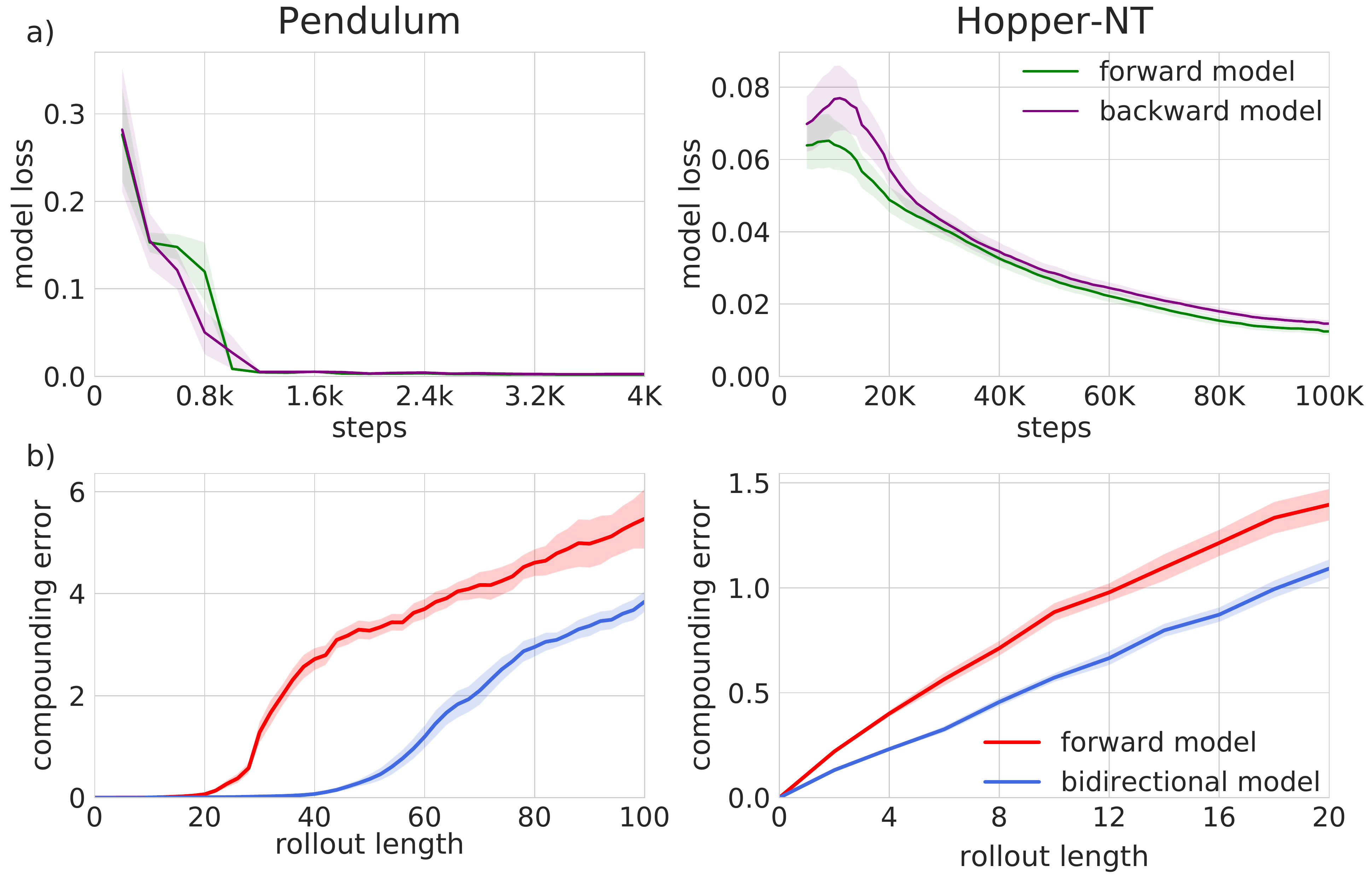}
    \vspace{-10pt}
    \caption{(a) Validation loss of forward/backward models. (b) Compounding error of the forward model and the bidirectional models in environment Pendulum (left) and Hopper-NT (right).}
    \vspace{-15pt}
  \label{fig:model error}
\end{figure}

%\begin{figure}[tb]
%    \centering
%    \includegraphics[width=0.48\textwidth]{figure/error.pdf}
%    \vspace{-10pt}
%    \caption{Model compounding error of the forward model and the bidirectional models in environment Pendulum (left) and Hopper-NT (right).}
%    \vspace{-15pt}
%  \label{fig:model error}
%\end{figure}

In this section, we first plot the validation loss of the  models in Figure~\ref{fig:model error}a, which can roughly represent the single-step prediction error. As the figure shows, the difficulty of learning forward/backward models is task-dependent. Then, we investigate the compounding error of traditional forward model and our bidirectional models when generating the same length simulated trajectories. We calculate the multi-step prediction error for evaluation. A similar validation error is also used in \citet{nagabandi2018neural}. More specifically, assume a real trajectory of length $2h$ is denoted as $(s_0,a_0,s_{1},\ldots,s_{2h})$. For the forward model, we sample from $s_0$ and generate forward rollouts $(\hat{s}_0,a_0,\hat{s}_{1},\ldots,\hat{s}_{2h})$ where $\hat{s}_{0}=s_{0}$ and for $i \geq 0$, $\hat{s}_{i+1}=p_{\theta}\left(\hat{s}_{i}, a_{i}\right)$. Then the corresponding compounding error is defined as
\begin{equation*}
\text{Error}_{\text{for}}=\frac{1}{2h} \sum_{i=1}^{2h}\left\|\hat{s}_{i}-s_{i}\right\|_{2}^{2}.
\end{equation*}

\begin{figure*}[ht]
\centering
\subfigure[Backward Policy]{
\includegraphics[width=0.235\textwidth]{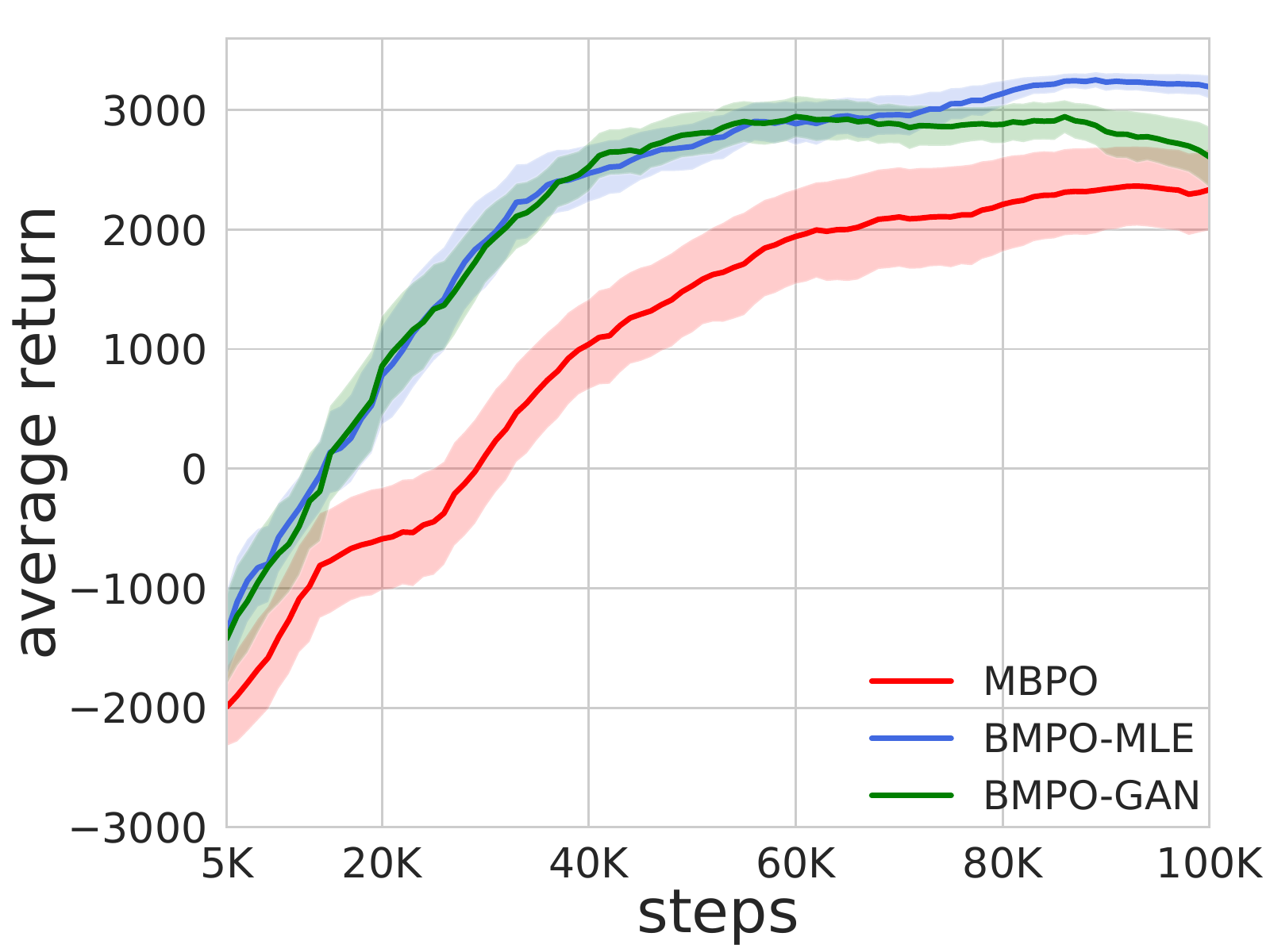}
\label{fig:backward policy}
}
\subfigure[Ablations]{
\includegraphics[width=0.235\textwidth]{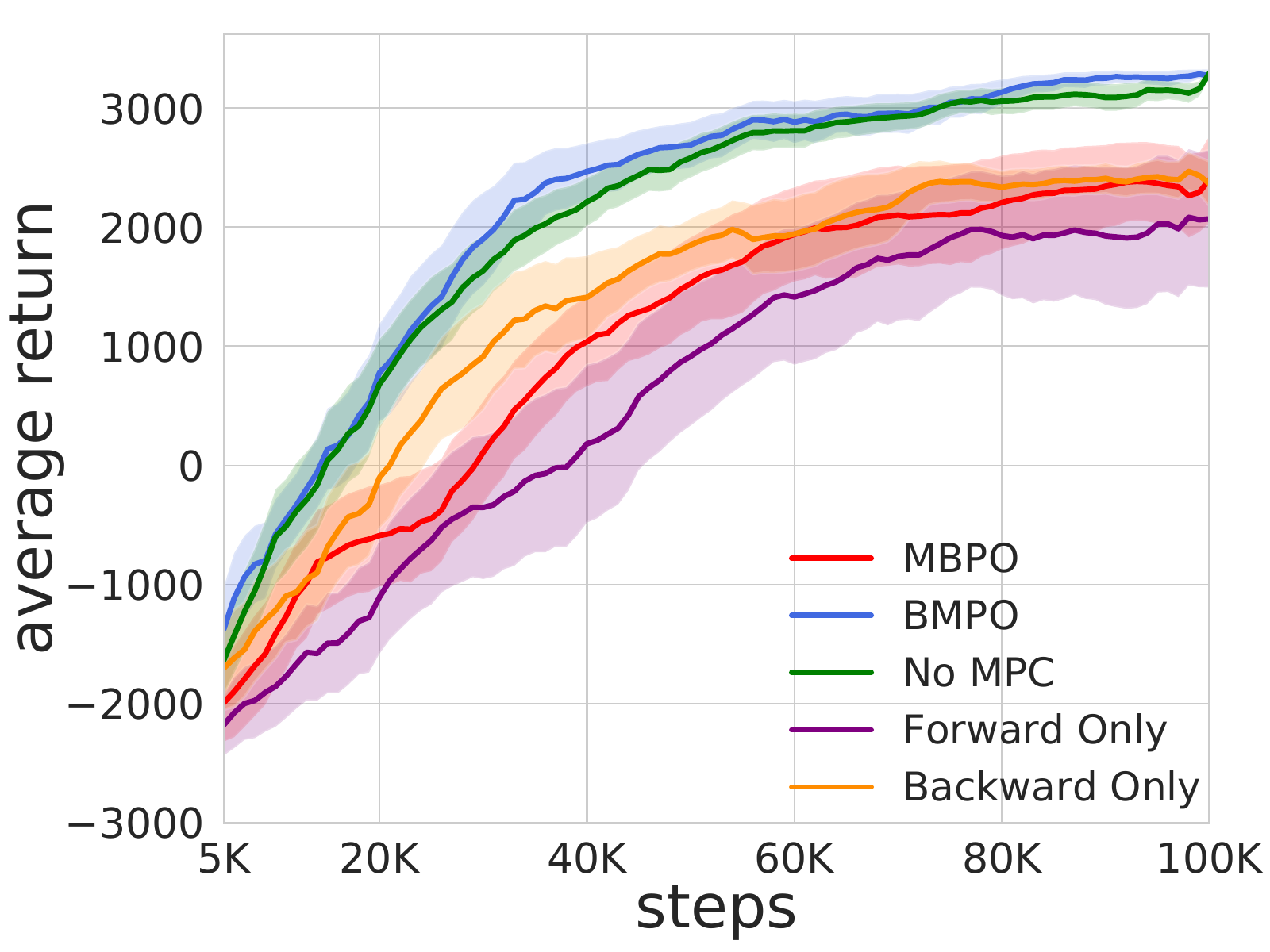}
\label{fig:ablations}
}
\subfigure[$\beta$ in Equation \ref{eq: state choose possibility}]{
\includegraphics[width=0.235\textwidth]{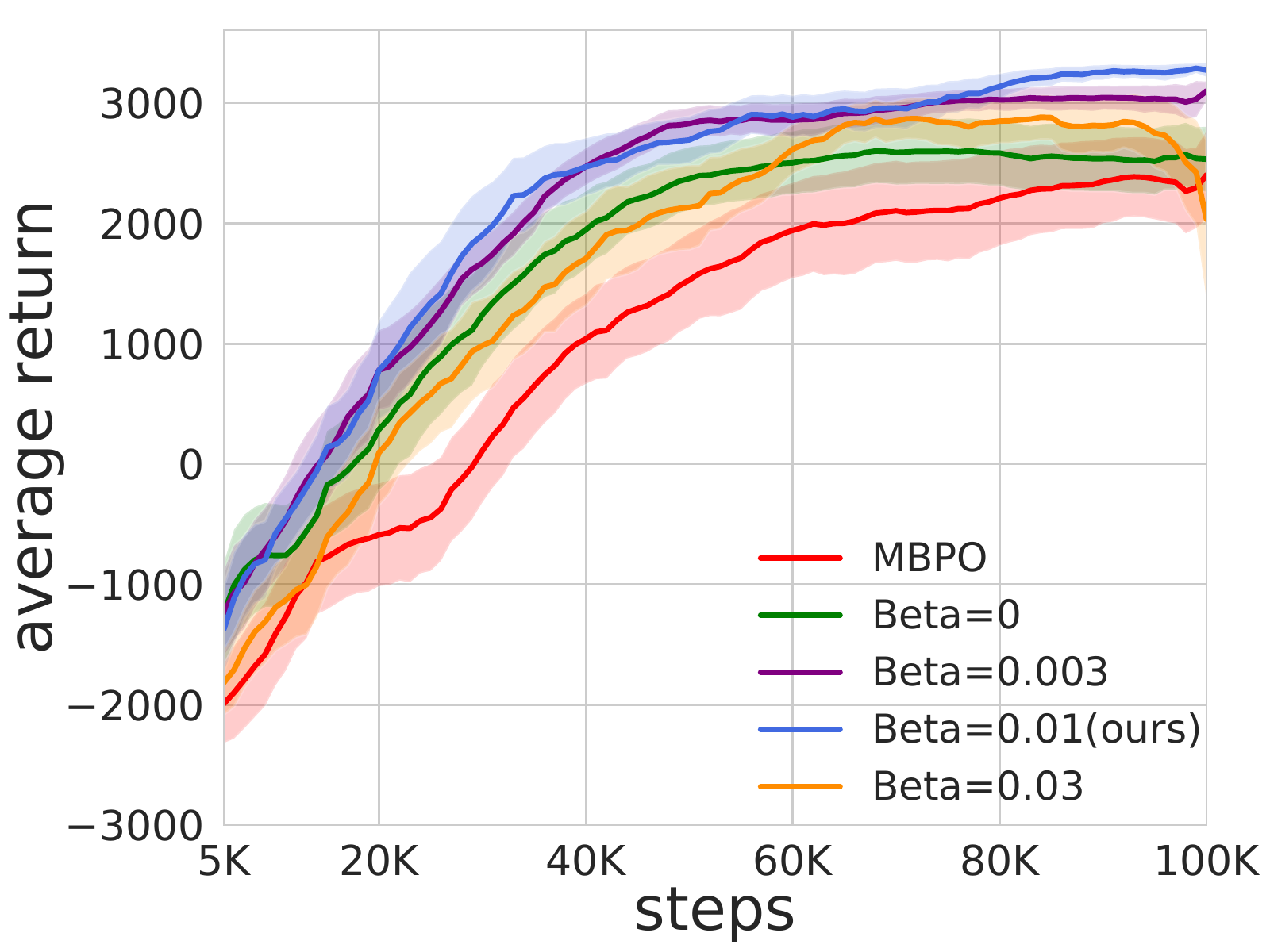}
\label{fig:beta}
}
\subfigure[Backward Length]{
\includegraphics[width=0.235\textwidth]{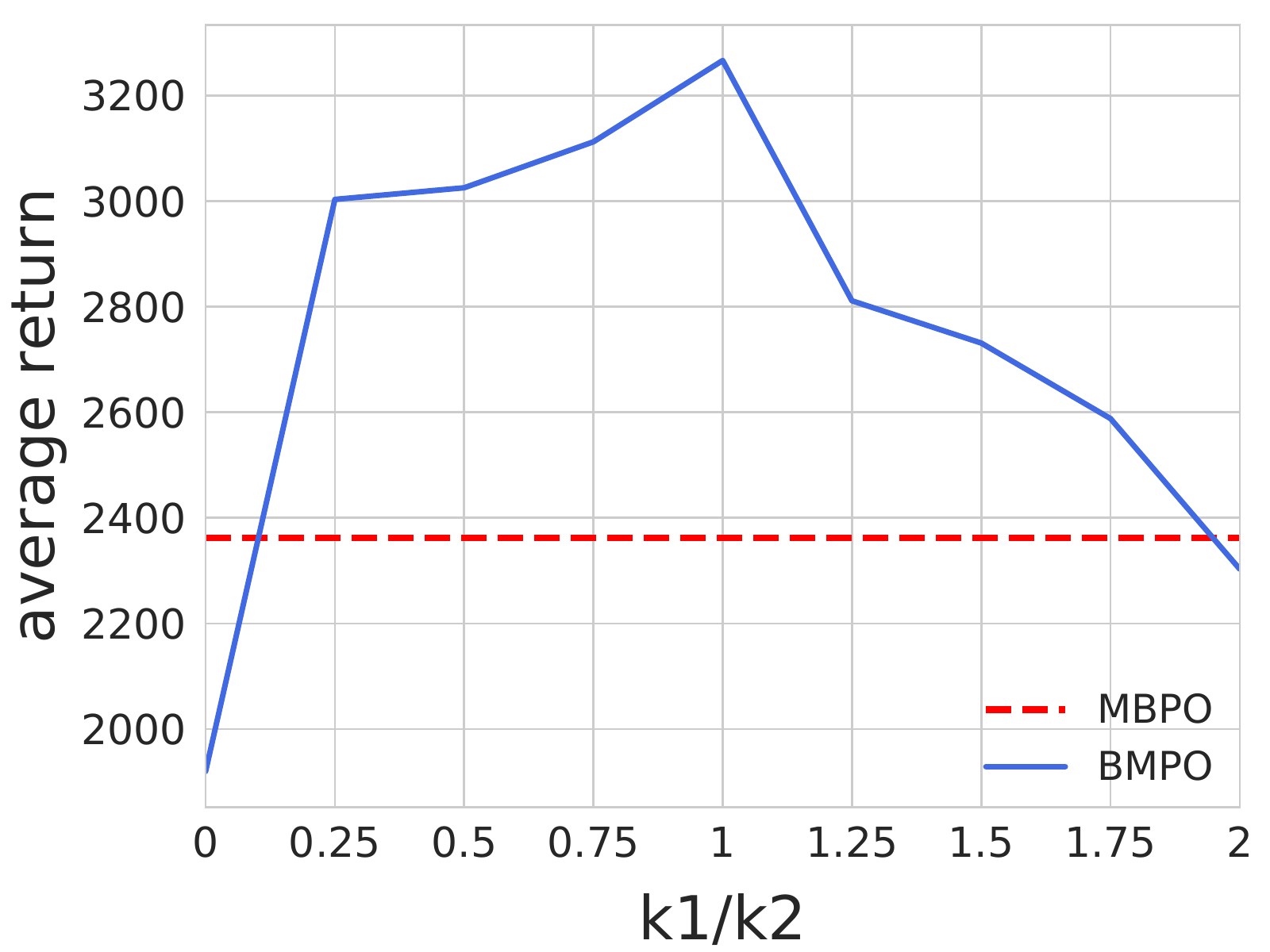}
\label{fig:backward length}
}
\caption{Design evaluation of our algorithm BMPO on task Hopper-NT. (a) Comparison of two heuristic design choices for the backward policy loss: MLE loss and GAN loss.
(b) Ablation study of three crucial components: forward model, backward model, and MPC.
(c) We study the sensitivity of our algorithm to the hyperparameter $\beta$ in Equation \ref{eq: state choose possibility}.
(d) Average return of the last 10 epochs over six trials with different backward rollout lengths $k_1$ and fixed forward length $k_2$.}
\label{fig:design evaluation}
\end{figure*}

Similarly, for the bidirectional models, suppose we sample from $s_{h}$ and generate both forward and backward rollouts of $h$ steps, compounding error is defined as
\begin{equation*}
\text{Error}_{\text{bi}}=\frac{1}{2h} \sum_{i=1}^{h}\left(\left\|\hat{s}_{h+i}-s_{h+i}\right\|_{2}^{2} + 
\left\|\hat{s}_{h-i}-s_{h-i}\right\|_{2}^{2} \right) ,
\end{equation*}
where $\hat{s}_{h}=s_{h}$ and for $i \geq 0$, $\hat{s}_{h+i+1}=p_{\theta}\left(\hat{s}_{h+i}, a_{h+i}\right)$ and $\hat{s}_{h-i-1}=q_{\theta^{\prime}}\left(\hat{s}_{h-i}, a_{h-i-1}\right)$.

We conduct experiments with different rollout length $2h$ and plot the results in Figure~\ref{fig:model error}b. We observe that employing bidirectional models can significantly reduce model compounding error, which is consistent with our intuition and theoretical analysis.

\subsection{Design Evaluation}
In this section, we evaluate the importance of each design decision for our overall algorithm. The results are demonstrated in Figure \ref{fig:design evaluation}.

\subsubsection{Backward Policy Design}
In Figure \ref{fig:backward policy}, we compare two heuristic design choices for the backward policy loss: MLE loss and GAN loss, which are described in detail in Section \ref{section:backward policy}. From the comparison, we notice a slight performance degradation in the late iterations of training when adopting GAN loss. One possible reason may be the instability and model collapse of GAN \cite{goodfellow2016nips}. More advanced methods to stabilize GAN's training can be incorporated into backward policy training, which remains as future work. As for other comparisons in this paper, we use the MLE loss for the backward policy as default.

\subsubsection{Ablation Study}
We further carry out an ablation study to characterize the importance of three main components of our algorithm: 1) no backward model to sample trajectory (Forward Only); \ 2) no forward model to sample trajectory (Backward Only); \ 3) no MPC when interacting with the environment (No MPC). The results are shown in Figure \ref{fig:ablations}. We find that ablating MPC decreases the performance slightly, and using only one forward model or backward model causes a more severe performance dropping. This further reveals the superiority of using bidirectional models. Notice that the performance of not using backward model is even worse than the vanilla MBPO, which means that sampling from high value states does not benefit the traditional forward model since the agent may only focus on how to act at high value states, but not care how to reach these states.

\subsubsection{Hyperparameter Study}
In this section, we investigate the sensitivity of BMPO to the hyperparameters.
First, we test BMPO with different hyperparameter \textbf{$\beta$} used in Equation \ref{eq: state choose possibility}. We vary \textbf{$\beta$} from $0$ to a relatively large number, where $\beta=0$ means random sampling, and larger $\beta$ means focusing more on high value states. The results are shown in Figure \ref{fig:beta}. We observe that up to a certain level, increasing \textbf{$\beta$} yields better performance while too large \textbf{$\beta$} can degrade the performance. This may be due to the fact that with a too large \textbf{$\beta$}, low value states are almost impossible to be chosen, and the value estimation at the beginning is difficult. Nevertheless, our algorithm with different \textbf{$\beta$} all outperform the baseline, which indicates the robustness to $\beta$.

Though it has been discovered that linearly increasing rollout length achieves excellent performance \cite{janner2019trust}, it remains a problem that how to choose the backward rollout length $k_1$ according to the forward length $k_2$. We fix $k_2$ to be the same as \citet{janner2019trust} and vary $k_1$ from 0 to $2k_2$. As is shown in Figure \ref{fig:backward length}, setting $k_1=k_2$ provides the best result, while too short and too long backward length $k_1$ are both detrimental. In practice, we use $k_1=k_2$ in most environments except Ant, where the backward model error is too large compared with forward model error. All hyperparameters settings are provided in Appendix \ref{appendix:Hyperparameter}.

\section{Conclusion}
In this work, we present a novel model-based reinforcement learning method, namely bidirectional model-based policy optimization (BMPO), using the newly introduced bidirectional models. We theoretically prove the advantage of bidirectional models by deriving a tighter return discrepancy upper bound of RL objective compared with only one forward model. Experimental results show that BMPO achieves better asymptotic performance and higher sample efficiency than previous state-of-the-art model-based methods on several benchmark continuous control tasks. For future work, we will investigate the usage of bidirectional models in other model-based RL frameworks and study how to leverage the bidirectional models better.

\section*{Acknowledgments}
The corresponding author Weinan Zhang thanks the support of "New Generation of AI 2030" Major Project 2018AAA0100900 and NSFC (61702327, 61772333, 61632017).

\bibliography{reference}

\onecolumn
\appendix

\section{BMPO Performance Guarantee}
\label{appendix:Theorem}

\begin{figure}[htb]
    \centering
    \includegraphics[width=0.7\textwidth]{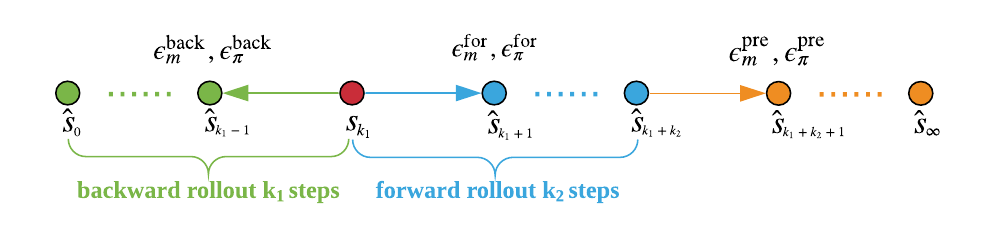}
    \caption{Bidirectional rollout.}
    \label{fig:bidirectional rollout}
\end{figure}

\begin{lemma}
\label{lemma:Bidirectional Branched Rollout Returns Bound}
(Bidirectional Branched Rollout Returns Bound). Let $\eta_1$, $\eta_2$ be the expected returns of two bidirectional branched rollouts. Out of the branch, we assume that the expected total variation distance between these two dynamics at each timestep $t$ is bounded as 
$\max _{t} E_{(s,a) \sim p_{1}^{t}(s,a)} D_{TV}\left(p_{1}^{\mathrm {pre }}\left(s^{\prime} | s, a\right) \| p_{2}^{\mathrm {pre }}\left(s^{\prime} | s, a\right)\right) \leq \epsilon_{m}^{\mathrm {pre }}$, similarly, the forward branch dynamic bounded as
$\max _{t} E_{(s,a) \sim p_{1}^{t}(s,a)} D_{TV}(p_{1}^{\mathrm {for }}\left(s^{\prime} | s, a\right) \| p_{2}^{\mathrm {for }}\left(s^{\prime} | s, a\right)) \leq \epsilon_{m}^{\mathrm {for }}$, and the backward branch dynamic bounded as 
$\max _{t} E_{(s^{\prime},a) \sim p_{1}^{t}(s^{\prime},a)} D_{TV}(p_{1}^{\mathrm {back }}\left(s | s^{\prime}, a\right) \| p_{2}^{\mathrm {back }}\left(s | s^{\prime}, a\right))\leq \epsilon_{m}^{\mathrm {back }}$. Likewise, the total variation distance of policy is bounded by $\epsilon_{\pi}^{\mathrm{pre}}$, $\epsilon_{\pi}^{\mathrm{for}}$ and $\epsilon_{\pi}^{\mathrm{back}}$, respectively (as Figure \ref{fig:bidirectional rollout} shows). Then the returns are bounded as
\begin{equation}
\begin{aligned}
\label{eq:Bidiractional_branch_return_Bound_append}
\left|\eta_{1}-\eta_{2}\right| \leq 2 r_{\max}\left[\frac{\gamma^{k_1+k_2+1}}{(1-\gamma)^{2}}\left(\epsilon_{m}^{\mathrm{pre}}+\epsilon_{\pi}^{\mathrm{pre}}\right)+\frac{\gamma^{k_1+k_2}}{1-\gamma}\epsilon_{\pi}^{\mathrm{pre}}+\frac{1-\gamma^{k_1}}{1-\gamma}\left(k_1\left(\epsilon_{m}^{\mathrm{back}}+\epsilon_{\pi}^{\mathrm{back}}\right)+\epsilon_{\pi}^{\mathrm{back}}\right)
\right.
\\ \left. +\frac{\gamma^{k_1}}{1-\gamma}\left(k_2\left(\epsilon_{m}^{\mathrm{for}}+\epsilon_{\pi}^{\mathrm{for}}\right)+\epsilon_{\pi}^{\mathrm{for}}\right)\right].
\end{aligned}
\end{equation}
\end{lemma}

\begin{proof}
Lemma \ref{lemma:Backward State Marginal Distance Bound}  and Lemma \ref{lemma:Forward State Marginal Distance Bound} imply that state marginal error at each timestep can be bounded by the divergence at the current timestep plus the state marginal error at the next (Lemma \ref{lemma:Backward State Marginal Distance Bound}), or previous (Lemma \ref{lemma:Forward State Marginal Distance Bound}) timestep. And by employing Lemma \ref{lemma:TVD Of Joint Distributions}, we can convert the (s,a) joint distribution to marginal distributions. Thus, letting $d_{1}(s, a)$ and $d_{2}(s, a)$ denote the state-action marginals, we can write: 
\\For $t \leq k_1$:
\begin{equation}
\begin{aligned}
    D_{T V}\left(d_{1}^{t}(s, a) \| d_{2}^{t}(s, a)\right) &\leq D_{T V}\left(d_{1}^{t}(s) \| d_{2}^{t}(s)\right) + \max _{s^{\prime}} D_{T V}\left({\pi}_{1}(a | s^{\prime}) \| {\pi}_{2}(a | s^{\prime})\right)
    \\& \leq (k_1-t)\left(\epsilon_{m}^{\mathrm{back}}+\epsilon_{\pi}^{\mathrm{back}}\right)+\epsilon_{\pi}^{\mathrm{back}} \leq k_1\left(\epsilon_{m}^{\mathrm{back}}+\epsilon_{\pi}^{\mathrm{back}}\right)+\epsilon_{\pi}^{\mathrm{back}}
\end{aligned}
\end{equation}
\\Similarly, for $k_1  \textless t \leq k_1+k_2$:
\begin{equation}
    D_{T V}\left(d_{1}^{t}(s, a) \| d_{2}^{t}(s, a)\right) \leq (t-k_1)\left(\epsilon_{m}^{\mathrm{for}}+\epsilon_{\pi}^{\mathrm{for}}\right)+\epsilon_{\pi}^{\mathrm{for}} \leq k_2\left(\epsilon_{m}^{\mathrm{for}}+\epsilon_{\pi}^{\mathrm{for}}\right)+\epsilon_{\pi}^{\mathrm{for}}
\end{equation}
\\And for $t \textgreater k_1+k_2$:
\begin{equation}
    D_{T V}\left(d_{1}^{t}(s, a) \| d_{2}^{t}(s, a)\right) \leq(t-k_1-k_2)\left(\epsilon_{m}^{\mathrm{pre}}+\epsilon_{\pi}^{\mathrm{pre}}\right)+k_2\left(\epsilon_{m}^{\mathrm{for}}+\epsilon_{\pi}^{\mathrm{for}}\right)+\epsilon_{\pi}^{\mathrm{pre}}+\epsilon_{\pi}^{\mathrm{for}}
\end{equation}

We can now bound the difference in occupancy measures by averaging the state marginal error over time, weighted by the discount:

$
\begin{aligned}  D_{T V}\left(d_{1}(s, a) \| d_{2}(s, a)\right) \leq &(1-\gamma) \sum_{t=0}^{\infty} \gamma^{t} D_{T V}\left(d_{1}^{t}(s, a) \| d_{2}^{t}(s, a)\right) 
\\ \leq &(1-\gamma) \sum_{t=0}^{k_1}\gamma^{t}\left(k_1\left(\epsilon_{m}^{\mathrm{back}}+\epsilon_{\pi}^{\mathrm{back}}\right)+\epsilon_{\pi}^{\mathrm{back}}\right) 
\\ &+(1-\gamma) \sum_{t=k_1}^{k_1+k_2}\gamma^{t}\left(k_2\left(\epsilon_{m}^{\mathrm{for}}+\epsilon_{\pi}^{\mathrm{for}}\right)+\epsilon_{\pi}^{\mathrm{for}}\right)
\\ &+(1-\gamma) \sum_{t=k_1+k_2}^{\infty}\gamma^{t}\left((t-k_1-k_2)\left(\epsilon_{m}^{\mathrm{pre}}+\epsilon_{\pi}^{\mathrm{pre}}\right)+k_2\left(\epsilon_{m}^{\mathrm{for}}+\epsilon_{\pi}^{\mathrm{for}}\right)+\epsilon_{\pi}^{\mathrm{pre}}+\epsilon_{\pi}^{\mathrm{for}}\right)
\\=& \left(k_1\left(\epsilon_{m}^{\mathrm{back}}+\epsilon_{\pi}^{\mathrm{back}}\right)+\epsilon_{\pi}^{\mathrm{back}}\right)\left(1-\gamma^{k_1}\right) + \left(k_2\left(\epsilon_{m}^{\mathrm{for}}+\epsilon_{\pi}^{\mathrm{for}}\right)+\epsilon_{\pi}^{\mathrm{for}}\right)\left(\gamma^{k_1}\right)
\\ &+\frac{\gamma^{k_1+k_2+1}}{1-\gamma}\left(\epsilon_{m}^{\mathrm{pre}}+\epsilon_{\pi}^{\mathrm{pre}}\right)+\gamma^{k_1+k_2} \epsilon_{\pi}^{\mathrm{pre}} \end{aligned}
$

Multiplying this bound by $\frac{2 r_{\max }}{1-\gamma}$ to convert the occupancy measure difference into a returns bound completes the proof.
\end{proof}

\begin{theorem}
\label{the:BMPO Return Discrepancy Upper Bound}
(BMPO Return Discrepancy Upper Bound) Assume that the expected total variation distance between the learned forward model $\hat{p}$ and the true dynamics $p$ at each timestep $t$ is bounded as 
$\max _{t} E_{(s,a) \sim \pi_{t}}\left[D_{T V}\left(p\left(s^{\prime} | s, a\right) \| \hat{p}\left(s^{\prime} | s, a\right)\right)\right] \leq \epsilon_{m}^{\mathrm{for}}$.
Similarly, the error of backward model $\hat{q}$ is bounded as 
$\max _{t} E_{(s^{\prime},a) \sim \pi_{t}}\left[D_{T V}\left(q\left(s | s^{\prime}, a\right) \| \hat{q}\left(s | s^{\prime}, a\right)\right)\right] \leq \epsilon_{m}^{\mathrm{back}}$ 
and the variation between current policy and the behavioral policy is bounded as 
$\max _{s} D_{T V}\left(\pi_{D}(a | s) \| \pi(a | s)\right) \leq \epsilon_{\pi}$.
Assume $\epsilon_{m}^{\mathrm{for}} \approx \epsilon_{m}^{\mathrm{back}} = \epsilon_{m}$ and $\epsilon_{\pi}^{\mathrm{back}} = 0$, then under a branched rollouts scheme with a backward branch length of $k_1$ and a forward branch length of $k_2$, the returns are bounded as:
\begin{equation}
\label{eq: BM_lower_Bound_appendix}
    \left|\eta[\pi]-\eta^{\mathrm {branch }}[\pi]\right| \leq 2 r_{\max }\left[\frac{\gamma^{k_1+k_2+1} \epsilon_{\pi}}{(1-\gamma)^{2}}+\frac{\gamma^{k_1+k_2} \epsilon_{\pi}}{(1-\gamma)}+\frac{max(k_1,k_2)}{1-\gamma}\left(\epsilon_{m}\right)\right].
\end{equation}
\end{theorem}
\begin{proof}
Using Lemma \ref{lemma:Bidirectional Branched Rollout Returns Bound}, out of the branch, we only suffer from error of executing old policy $\pi_{D}$, so, set $\epsilon_{\pi}^{\mathrm{pre}} = \epsilon_{\pi}$ and $\epsilon_{m}^{\mathrm{pre}} = 0$. Then in the branched rollout, we execute current policy, so the only error comes from using the learned model to simulate. Set $\epsilon_{\pi}^{\mathrm{for}} = \epsilon_{\pi}^{\mathrm{back}} = 0$ and $\epsilon_{m}^{\mathrm{for}} = \epsilon_{m}^{\mathrm{back}} = \epsilon_{m}$. Plugging these in  Lemma B.1 we can get:
\begin{equation}
\begin{aligned}  
\left|\eta[\pi]-\eta^{\mathrm {branch }}[\pi]\right| &\leq 2 r_{\max }\left[\frac{\gamma^{k_1+k_2+1} \epsilon_{\pi}}{(1-\gamma)^{2}}+\frac{\gamma^{k_1+k_2} \epsilon_{\pi}}{(1-\gamma)}+\frac{k_1(1-\gamma^{k_1})+k_2(\gamma^{k_1})}{1-\gamma}\left(\epsilon_{m}\right)\right]
\\&\leq 2 r_{\max }\left[\frac{\gamma^{k_1+k_2+1} \epsilon_{\pi}}{(1-\gamma)^{2}}+\frac{\gamma^{k_1+k_2} \epsilon_{\pi}}{(1-\gamma)}+\frac{max(k_1,k_2)(1-\gamma^{k_1}+\gamma^{k_1}))}{1-\gamma}\left(\epsilon_{m}\right)\right]
\\&\leq 2 r_{\max }\left[\frac{\gamma^{k_1+k_2+1} \epsilon_{\pi}}{(1-\gamma)^{2}}+\frac{\gamma^{k_1+k_2} \epsilon_{\pi}}{(1-\gamma)}+\frac{max(k_1,k_2)}{1-\gamma}\left(\epsilon_{m}\right)\right]
\end{aligned}
\end{equation}
\end{proof}

\section{Useful Lemmas}
In this section, we give proofs of the lemmas used before.
\label{appendix:Lemma}

\begin{lemma}
\label{lemma:Backward State Marginal Distance Bound}
(Backward State Marginal Distance Bound). Suppose the expected total variation distance between two backward dynamics is bounded as $\max _{t} E_{(s^{\prime},a) \sim p_{1}^{t}}\left[D_{T V}\left(p_{1}\left(s | s^{\prime}, a\right) \| p_{2}\left(s | s^{\prime}, a\right)\right)\right] \leq \epsilon_{m}^{\mathrm{back}}$ and the backward policy divergences are bounded as $\max _{s^{\prime}} D_{T V}\left(\pi_{1}(a | s^{\prime}) \| \pi_{2}(a | s^{\prime})\right) \leq \epsilon_{\pi}^{\mathrm{back}}$. Then the state marginal distance at timestep $t$ can be bounded as:
\begin{equation}
    D_{T V}\left(p_{1}^{t}(s) \| p_{2}^{t}(s)\right) \leq \epsilon_{m}^{\mathrm{back}} + \epsilon_{\pi}^{\mathrm{back}} + D_{T V}\left(p_{1}^{t+1}(s) \| p_{2}^{t+1}(s)\right).
\end{equation}
\end{lemma}
\begin{proof}
Let the total variation distance of state at time $t$ be denoted as $\epsilon_{t} = D_{T V}\left(p_{1}^{t}(s) \| p_{2}^{t}(s)\right)$.
$
\begin{aligned}\left|p_{1}^{t}(s)-p_{2}^{t}(s)\right| =&\Big|\sum_{s^{\prime},a} p_{1}\left(s_{t}=s | s^{\prime},a\right) p_{1}^{t+1}\left(s^{\prime},a\right)-p_{2}\left(s_{t}=s | s^{\prime},a\right) p_{2}^{t+1}\left(s^{\prime},a\right)\Big| 
\\  \leq & \sum_{s^{\prime},a}\left| p_{1}\left(s_{t}=s | s^{\prime},a\right) p_{1}^{t+1}\left(s^{\prime},a\right)-p_{2}\left(s_{t}=s | s^{\prime},a\right) p_{2}^{t+1}\left(s^{\prime},a\right)\right| 
\\ =& \sum_{s^{\prime},a}| p_{1}\left(s_{t}=s | s^{\prime},a\right) p_{1}^{t+1}\left(s^{\prime},a\right) - p_{2}\left(s_{t}=s | s^{\prime},a\right) p_{1}^{t+1}\left(s^{\prime},a\right) 
\\&+ p_{2}\left(s_{t}=s | s^{\prime},a\right) p_{1}^{t+1}\left(s^{\prime},a\right)
-p_{2}\left(s_{t}=s | s^{\prime},a\right) p_{2}^{t+1}\left(s^{\prime},a\right)|
\\  \leq& \sum_{s^{\prime},a} p_{1}^{t+1}\left(s^{\prime},a\right)\left|p_{1}\left(s | s^{\prime},a\right)-p_{2}\left(s | s^{\prime},a\right)\right|+p_{2}\left(s | s^{\prime},a\right)\left|p_{1}^{t+1}\left(s^{\prime},a\right)
-p_{2}^{t+1}\left(s^{\prime},a\right)\right| 
\\ =&E_{s^{\prime},a \sim p_{1}^{t+1}}\big[\left|p_{1}\left(s | s^{\prime},a\right)-p_{2}\left(s | s^{\prime},a \right) \right|\big]+\sum_{s^{\prime},a} p_{2}\left(s | s^{\prime},a\right)\left| p_{1}^{t+1}\left(s^{\prime},a\right)-p_{2}^{t+1}\left(s^{\prime},a\right) \right|\end{aligned}
$

$
\begin{aligned} \epsilon_{t}=D_{T V}\left(p_{1}^{t}(s) \| p_{2}^{t}(s)\right) &=\frac{1}{2} \sum_{s}\left|p_{1}^{t}(s)-p_{2}^{t}(s)\right| 
\\ &\leq\frac{1}{2} \sum_{s}\Big( E_{s^{\prime},a \sim p_{1}^{t+1}}\left[\left|p_{1}\left(s | s^{\prime},a\right)-p_{2}\left(s | s^{\prime},a \right) \right|\right]+\sum_{s^{\prime},a} p_{2}\left(s | s^{\prime},a\right)\left| p_{1}^{t+1}\left(s^{\prime},a\right)-p_{2}^{t+1}\left(s^{\prime},a\right) \right| \Big)
\\ &= E_{s^{\prime},a \sim p_{1}^{t+1}}\left[ D_{T V}\left(p_{1}\left(s | s^{\prime}, a\right) \| p_{2}\left(s | s^{\prime}, a\right)\right)
\right]+D_{T V}\left(p_{1}^{t+1}\left(s^{\prime},a\right) \| p_{2}^{t+1}\left(s^{\prime},a\right)\right) 
\\ &\leq \epsilon_{m}^{\mathrm{back}} + D_{T V}\left(p_{1}^{t+1}\left(s^{\prime}\right) \| p_{2}^{t+1}\left(s^{\prime}\right)\right) + \max _{s^{\prime}} D_{T V}\left(p_{1}(a | s^{\prime}) \| p_{2}(a | s^{\prime})\right) 
\\ &=\epsilon_{m}^{\mathrm{back}} + \epsilon_{\pi}^{\mathrm{back}} + D_{T V}\left(p_{1}^{t+1}(s) \| p_{2}^{t+1}(s)\right) \end{aligned}
$
\end{proof}

\begin{lemma}
\label{lemma:Forward State Marginal Distance Bound}
(Forward State Marginal Distance Bound) (\cite{janner2019trust}, Lemma B.2, B.3). Suppose the expected TVD between two forward dynamics is bounded as $\max _{t} E_{(s,a) \sim p_{1}^{t}}\left[D_{T V}\left(p_{1}\left(s^{\prime} | s, a\right) \| p_{2}\left(s^{\prime} | s, a\right)\right)\right] \leq \epsilon_{m}^{\mathrm{for}}$ and the forward policy divergences are bounded as $\max _{s^{\prime}} D_{T V}\left(\pi_{1}(a | s) \| \pi_{2}(a | s)\right) \leq \epsilon_{\pi}^{\mathrm{for}}$. Then the state marginal distance at timestep $t$ can be bounded as:
\begin{equation}
    D_{T V}\left(p_{1}^{t}(s) \| p_{2}^{t}(s)\right) \leq \epsilon_{m}^{\mathrm{for}} + \epsilon_{\pi}^{\mathrm{for}} + D_{T V}\left(p_{1}^{t-1}(s) \| p_{2}^{t-1}(s)\right).
\end{equation}
\end{lemma}

\begin{lemma}
\label{lemma:TVD Of Joint Distributions}
(TVD Of Joint Distributions) (\cite{janner2019trust}, Lemma B.1). Suppose we have two distributions $p_{1}(x, y)=p_{1}(x) p_{1}(y | x)$ and $p_{2}(x, y)=p_{2}(x) p_{2}(y | x)$. We can bound the total variation distance of the joint distributions as:
\begin{equation}
    D_{T V}\left(p_{1}(x, y) \| p_{2}(x, y)\right) \leq D_{T V}\left(p_{1}(x) \| p_{2}(x)\right)+\max _{x} D_{T V}\left(p_{1}(y | x) \| p_{2}(y | x)\right).
\end{equation}
\end{lemma}

\section{Environment Settings}
\label{appendix:Env}
In this section, we provide a comparison of the environment settings used in our experiments. Among them, 'Hopper-NT' and 'Walker2d-NT' refer to the settings in \citet{langlois2019benchmarking} and others are the standard version.

\begin{table}[htb]
\centering
\caption{Observation and action dimension, and task horizon of the environments used in our experiments.}
\vskip 0.10in
\begin{tabular}{@{}c|c|c|c@{}}
\toprule
Environment Name & Observation Space Dimension & Action Space Dimension & Steps Per Epoch\\ \midrule
Pendulum & 3 & 1 & 200\\
Hopper & 11 & 3 & 1000\\
Hopper-NT & 11 & 3 & 1000\\
Walker2d & 17 & 6 & 1000\\
Walker2d-NT & 17 & 6 & 1000\\
Ant & 27 & 8 & 1000\\
\bottomrule
\end{tabular}
\end{table}

\begin{table}[htb]
\centering
\caption{Reward function and termination states condition of the environments used in our experiments. $\theta_t$ denotes the joint angle, $x_t$ denotes the position in x direction, $a_t$ denotes the action control input, and $z_t$ denotes the height.}
\vskip 0.10in
\begin{tabular}{@{}c|c|c@{}}
\toprule
Environment Name & Reward Function & Termination States Condition \\ \midrule
Pendulum & $-{\theta}_{t}^{2}-0.1\dot{\theta}_{t}^{2}-0.001\left\|a_{t}\right\|_{2}^{2}$ & None \\
Hopper & $\dot{x}_{t}-0.001\left\|a_{t}\right\|_{2}^{2} + 1$ & $z_{t} \leq 0.7$ or $\theta_{t} \geq 0.2$ \\
Hopper-NT & $\dot{x}_{t}-0.1\left\|a_{t}\right\|_{2}^{2}-3.0 \times\left(z_{t}-1.3\right)^{2} + 1$ & None \\
Walker2d & $\dot{x}_{t}-0.001\left\|a_{t}\right\|_{2}^{2} + 1$ & $z_{t} \leq 0.8$ or $z_{t} \geq 2.0$ or $|\theta_{t}| \geq 1.0$ \\
Walker2d-NT & $\dot{x}_{t}-0.1\left\|a_{t}\right\|_{2}^{2}-3.0 \times\left(z_{t}-1.3\right)^{2} + 1$ & None \\
Ant & $\dot{x}_{t}-0.5\left\|a_{t}\right\|_{2}^{2} + 1$ & $z_{t} \leq 0.2$ or $z_{t} \geq 1.0$ \\
\bottomrule
\end{tabular}
\end{table}

\newpage

\section{Hyperparameters}
\label{appendix:Hyperparameter}
\newcommand{\tabincell}[2]{\begin{tabular}{@{}#1@{}}#2\end{tabular}}
\begin{table}[htb]
\centering
\caption{Hyperparameter settings for BMPO. $x \to y$ over epochs $a \to b$ means clipped linear function, \textit{i.e.} for epoch e, $f(e) = clip((x+\frac{e-a}{b-a} \cdot(x-y)),x,y)$.  Other hyperparameters not listed here are the same as those in MBPO \cite{janner2019trust}.}
\vskip 0.10in
\begin{tabular}{@{}c|c|c|c|c|c@{}}
\toprule
Environment Name &  $k_1$ & $k_2$ & $\beta$ & MPC Horizon & Epochs\\ \midrule
Pendulum & \tabincell{c}{$1 \to 5$ over \\ epochs $1 \to 5$} & \tabincell{c}{$1 \to 5$ over \\ epochs $1 \to 5$} & \tabincell{c}{$0.01 \to 0$ over \\ epochs $0 \to 10$} & 6 & 20\\ \midrule
Hopper & \tabincell{c}{$1 \to 15$ over \\ epochs $20 \to 150$} & \tabincell{c}{$1 \to 15$ over \\ epochs $20 \to 150$} & \tabincell{c}{$0.004 \to 0.003$ over \\ epochs $20 \to 30$} & 6 & 100\\ \midrule
Hopper-NT & \tabincell{c}{$1 \to 15$ over \\ epochs $20 \to 150$} & \tabincell{c}{$1 \to 15$ over \\ epochs $20 \to 150$} & 0.01 & 6 & 100\\ \midrule
Walker2d & 1 & 1 & \tabincell{c}{$0.01 \to 0$ over \\ epochs $0 \to 100$} & 1 & 200\\ \midrule
Walker2d-NT & 1 & 1 & 0.01 & 0 & 200\\ \midrule
Ant & 1 & \tabincell{c}{$1 \to 25$ over \\ epochs $20 \to 100$} & 0.003 & 0 & 300\\
\bottomrule
\end{tabular}
\end{table}

\section{Computing Infrastructure}
\label{appendix:Computing Infrastructure}
In this section, we provide a description of the computing infrastructure used to run all the experiments in Table~\ref{tab:infra}. We also show the computation time comparison between our algorithm and the MBPO baseline in Table~\ref{tab:comp-time}.

\begin{table}[htb]
\centering
\caption{Computing infrastructure.}\label{tab:infra}
\vskip 0.10in
\begin{tabular}{@{}c|c|c@{}}
\toprule
CPU & GPU & Memory\\ \midrule
AMD2990WX & RTX2080TI$\times$4 & 256GB\\
\bottomrule
\end{tabular}
\end{table}

%\vspace{10pt}

\begin{table}[!htb]
\centering
\caption{Computation time in hours for one experiment.}\label{tab:comp-time}
\vskip 0.10in
\begin{tabular}{@{}c|c|c|c|c|c|c@{}}
\toprule
 & Pendulum & Hopper & Hopper-NT & Walker2d & Walker2d-NT & Ant\\ \midrule
BMPO& 0.49 & 16.34 & 17.98 & 27.24 & 27.34 & 71.51\\
MBPO& 0.41 & 10.33 & 11.12 & 22.26 & 21.32 & 57.42\\
\bottomrule
\end{tabular}
\end{table}

\bibliographystyle{icml2020}

\end{document}